%% file: main-V2.tex
\documentclass{article}

\usepackage{amsmath, amssymb, amsfonts, mathrsfs, multicol}
\usepackage{amsthm}
\usepackage{mathtools}
\usepackage{color}
\usepackage{graphicx}
\usepackage{tikz-cd}
\usepackage{enumitem}

\theoremstyle{plain}
\newtheorem{theorem}{Theorem}[section]
\newtheorem{lemma}[theorem]{Lemma}

\theoremstyle{definition}
\newtheorem{definition}[theorem]{Definition}

\theoremstyle{remark}

\newcommand{\XQ}{X/{\sim}}
\newcommand*{\vv}[1]{\overrightarrow{#1}}

\begin{document}

\input{Title-Abstract}


\input{Intro}


\input{Related-Work}


\input{Results}

\input{Conclusions}

\input{Acknowledgements}

\bibliographystyle{unsrt}
\bibliography{am}

\end{document}

%% file: Title-Abstract.tex
\begin{center}
\Large 
Dimension Reduction Using Active Manifolds\\
\normalsize
Team Bridges, Summer 2016\\
Robert A. Bridges, Ph.D., Christopher R. Felder, Chelsey R. Hoff\\
\end{center}

%% file: Intro.tex
\section{Introduction}

Scientists and engineers rely on accurate mathematical models to quantify the objects of their studies, which are often high-dimensional. Unfortunately, high-dimensional models are inherently difficult, i.e. when observations are sparse or expensive to determine. One way to address this problem is to approximate the original model with fewer input dimensions. Our project goal was to recover a function $f$ that takes $n$ inputs and returns one output, where $n$ is potentially large. For any given $n$-tuple, we assume that we can observe a sample of the gradient and output of the function but it is computationally expensive to do so. This project was inspired by an approach known as Active Subspaces, which works by linearly projecting to a linear subspace where the function changes most on average. Our research gives mathematical developments informing a novel algorithm for this problem. Our approach, Active Manifolds, increases accuracy by seeking nonlinear analogues that approximate the function. The benefits of our approach are eliminated unprincipled parameter, choices, guaranteed accessible visualization, and improved estimation accuracy.

%% file: Related-Work.tex
\section{Related Work}
Dimension reduction, broadly defined, is the mapping of potentially high dimensional data to a lower dimensional space. Dimension reduction techniques can be categorized into two main categories, projective methods and manifold modeling  \cite{burges2010dimension}. Dimension reduction techniques are widely used across many domains to analyze high-dimensional models or high-dimensional data sets because they allow important  low-dimensional features to be extracted and allow for data visualization. The most commonly known and used projective method is Principal Component Analysis (see \cite{PCA}). The method that inspired our work, Active Subspaces, can also be considered a projective method. The Nystr\"{o}m method (see \cite{nystrommethod}) and related variations rely on eigenvalue problems and compromise the bulk of manifold modeling techniques. Our method, Active Manifolds, departs from the use of projective and spectral methods but is a manifold modeling method. 


\subsection{Active Subspaces}
We chose to study Active Subspaces because it is a dimension reduction technique that reduces the dimension of the input space while respecting the output, its applicability to a wide range of functions ($C^1(\mathbb{R}^n, \mathbb{R})$), and because of its accessibility to scientists and engineers with a limited mathematical background. The Active Subspaces method finds lower-dimensional subspaces of the domain by finding the directions in which the function changes the most on average. The Active Subspaces method has two main limitations. First, many functions do not admit a linear active subspace, e.g. $f(x,y) = x^2 + y^2$. Second, the linearity of active subspaces and projections is restrictive and can increase estimate error.

\noindent Below is a brief description of the Active Subspaces algorithm.

\begin{enumerate}
\item Sample $\nabla f$ at $N$ random points $x \in U$
\item Find the directions in which $f$ changes the most on average, \emph{\textbf{the active subspace}}. This is done by computing the eigenvalue decomposition of the matrix
$$\mathbf{C} = \frac{1}{N} \sum_{i=1}^{N} \nabla_x f_i \nabla_x f_{i}^{T} = \mathbf{W}\mathbf{\Lambda}\mathbf{W}^{T}     $$
\item Perform regression to estimate $f$ along the active subspace to obtain $f \approx \hat{f}$ (this requires sampling $f$ at random points $x \in U$).
\item Given a new point $p \in U$, project $p$ to the active subspace and use $\hat{f}$ to obtain the value $f(p) \approx \hat{f}(p)$.
\end{enumerate}

%% file: Results.tex
\section{Results}

\subsection{Theory}

Recall that arc length of a  $C^1$ curve $x(t) : [0, 1] \rightarrow \mathbb{R}^n$ is given by
$$ \int_{0}^{1} |x'(t)| \ dt \ .$$

\noindent Let $\mathcal{U} \subseteq \mathbb{R}^n$, where $n < \infty$. Assume $f: \mathbb{R}^n \rightarrow \mathbb{R}$ is $C^1$. We seek
$$\arg \max \int_{0}^{1}  \langle \nabla f(x(t)), x'(t) \rangle \ dt.$$ \\
over all $C^1$ functions $x(t) : [0, 1] \rightarrow \mathbb{R}^n$, and $\|x'\| = 1$ (constant speed).

\noindent Notice the integrand can be expressed as
$$ \langle \nabla f(x(t)), x'(t) \rangle  = \|\nabla f(x(t)) \| \ \| x'(t) \| \cos{\theta} $$
where $\theta$ is the angle between $\nabla f(x(t))$ and $x'(t)$. Trivially, this quantity is maximal when $\theta = 0$, indicating that $\nabla f(x(t))$ and $x'(t)$ are collinear and point in the same direction. Thus,
\begin{equation}
\label{am-condition}
x'(t)= \frac{\nabla f(x(t))}{\| \nabla f(x(t)) \|} \ .
\end{equation}

\begin{definition}
\label{am-def}
Let $U \subseteq \mathbb{R}^n$ and $f : U \xrightarrow[]{C^1} \mathbb{R}$ and $\mathcal{M} \subseteq U$. We say that $\mathcal{M}$ is an $\textbf{active manifold}$ of $f$ if and only if, for all charts $(\mathcal{M}, \Phi)$ on $U$, condition (\ref{am-condition}) is satisfied when $\Phi = x^{-1}(t)$.
\end{definition}

\begin{lemma}
\label{am-prop-1}
Given $f : U \xrightarrow[]{C^1} \mathbb{R}$ and an initial value $x_{0} \in U$, there exists a unique solution $x(t)$ to the system of first-order differential equations described in (\ref{am-condition}).
\end{lemma}

\begin{proof}
Assume the region $U$ is compact and convex. Since $f$ is $C^1$, $\nabla f(x(t))$ satisfies the Lipschitz condition
$$| \nabla f(x(t)) - \nabla f(\hat{x}(t)) | \leq L |x(t) -  \hat{x}(t)|$$ for $x(t), \hat{x}(t) \in U$ and some Lipschitz constant $L$. These conditions are sufficient for the existence and uniqueness of a solution $x(t)$ to (\ref{am-condition}) for a given initial value  $x_{0} \in U$ (see Chapter 6, Theorem 1 from \cite{birkhoff1969ODE}).
\end{proof}

%
%
%
%
%

\begin{multicols}{2}
\noindent For the following theorem, let
\begin{itemize}
\item $M$ = range$(x)$
\item $c$ be a fixed critical point of $f$
\item $X$ be the deleted attracting basin of $c$
\item $x \sim y \iff f(x) = f(y)$
\item $[x] = \{ y \in X : f(x) = f(y) \}$
\item $\pi : X \rightarrow \XQ$\\
\end{itemize}
\columnbreak
\begin{tikzcd}[row sep=huge, column sep = huge]
 	& X \arrow[rd, "f"] \arrow[d, "\pi"] & \\[1.9cm]
	\mathbb{R} \arrow[ur, "x(t)"] \arrow[r, "\pi \circ x"] & \XQ \arrow[r, "\tilde{f}"] & \mathbb{R}
\end{tikzcd}
\centering Commutative diagram for Theorem \ref{FToAM}
\end{multicols}

\begin{theorem}
\label{FToAM}
\hspace*{1 em}
\begin{enumerate}[label=(\roman*)]
\item If $x(t)$ is a solution to (\ref{am-condition}), then $M$ is a 1-dim. submanifold of $\mathbb{R}^n$.
\item $\XQ$ is a manifold.
\item If $x_0 \in X$ and $x(t)$ is a solution to (\ref{am-condition}) then $M$ imbeds into the manifold $\XQ$.
\end{enumerate}
\end{theorem}
\begin{proof}
\hspace*{1 em}
\begin{enumerate}[label=(\roman*)]
\item Realize $(f \rvert_{M}, M)$ as the single chart for $M$ induced by $f$, thus $M$ is a 1-dimensional submanifold of $\mathbb{R}^n$.
\item Realize $\tilde{f}: \XQ \rightarrow \mathbb{R}$  given by $\tilde{f}([x]) = f(x)$ is continuous so $\XQ$ is a manifold with a single chart $(\tilde{f}, \XQ)$.
\item Realize $\pi\rvert_M : M \rightarrow \pi(M)$ is a bijection. It follows that $\pi\rvert_{M}$ is a diffeomorphism from $M$ to $\pi(M)$ since charts on $M$ and $\XQ$ are induced by $f$.
\end{enumerate}
\end{proof}

Further, the Implicit Function Theorem implies that $\{f = const.\}$ is a $(n-1)$-dimensional manifold and orthogonally intersects $\{x(t)\}$. We refer to such a manifold as an $\textit{\textbf{Active Manifold}}$ (denoted $AM$). Thus, we can realize an $AM$ by a numerical solution to (\ref{am-condition}).

\subsection{Active Manifold Algorithm Description}

The Active Manifolds algorithm has three main procedures:
\begin{enumerate}
\item Building the Active Manifold
\item Approximating the function of interest, $f \approx \hat{f}$
\item Projecting a point of interest to the Active Manifold
\end{enumerate}

\subsubsection{Building the Active Manifold}
For a given function, $f : \mathcal{U} \xrightarrow[]{C^1} \mathbb{R}$ where $\mathcal{U} \subseteq \mathbb{R}^n$, we describe below a process to build a corresponding active manifold. The active manifold will be a one-dimensional curve in the hypercube $[-1,1]^{n}$ that moves from a local minimum to a local maximum.  We define a grid with spacing size $\epsilon$ then compute $\nabla f$ at each grid point. To build the active manifold, we use a modified gradient ascent/descent scheme with a nearest neighbor search.
\begin{enumerate}

\item Construct an $n$-dimensional grid with spacing size $\epsilon$.
\begin{enumerate}[label*=\arabic*.]
\item Compute $\nabla f$ at each grid point.
\item Normalize $\nabla f$ samples.
\end{enumerate}
\item  Given an initial starting point $x_0 \in U$, use a gradient ascent/descent scheme with a nearest neighbor search to find a numerical solution to $$x'(t) = \frac{\nabla f(x(t))}{||\nabla f(x(t))||}$$ with the samples from step 1. The set $\{x(t) \}_t$ is an active manifold on $[-1,1]^n$.
\begin{enumerate}[label*=\arabic*.]
\item While the active manifold builds, save the number of steps and functional values corresponding to the closest grid point at each step as ordered lists $\mathbf{S}:= [ 0,\ldots,step_k,\ldots, step_n ]$ and $\mathbf{Z}:=[ z_0,\ldots,z_k,\ldots, z_n ]$.
\item Scale $\mathbf{S}$ by $n$ so that $\mathbf{S}:= [ 0,\ldots,\frac{step_k}{n},\ldots, 1]$
\end{enumerate}
\end{enumerate}

\subsubsection{One-Dimensional Function Approximation}
To obtain a one-dimensional approximation $f \approx \hat{f}$, perform regression on the points of $M$. A major benefit of our method is that the set $\mathbf{M}: = \{(s,z): s_i \in \mathbf{S}, z_j \in \mathbf{Z} \text{ for } i = j \}$ can be easily plotted and serves as a visual aid to help choose a best fit model.

\subsubsection{Traversing the Level Set}
Given a point $p$, we would like to find $\hat{f}(q)$, where $q \in [p]$ and $q \in \textbf{M}$. This requires an iterative process that uses the orthogonal directions of $\nabla f(p)$ to travel along the level set corresponding to $f(p)$, until we intersect the active manifold. For the following algorithm, we assume that $\nabla f$ has been normalized and tolerance $\epsilon$ and step size $\delta$ have been selected.
\begin{enumerate}
\item Given a point $p$, find $\arg \min \| p - m \|$ for $m \in \mathcal{M}$ (closest point on the manifold to $p$)$^1$
\item Construct a vector from $p$ to $m$, $\vv{u}$.
\item Find $\vv{v} = Proj_{(\nabla f(p))^{\perp}}(\vv{u})$.$^2$
\item While  $\|p - m \| < \epsilon$, let  $p = p + \epsilon p$.
\item Parameterize the line segment, $M(s)$, between $m$ and $m_+$, where $m_+$ is the next closest point on the manifold to $p$.$^3$
\item Determine $s$ such that $M(s)  - \vv{u}$ is orthogonal to $\nabla f(p)$.$^4$
\item Determine $t$ such that $M(s) = x(t)$.
\item Evaluate $\hat{f}(t) \approx f(p)$.
\end{enumerate}

\begin{figure}[h!]
\centering
\includegraphics[width = .5\linewidth]{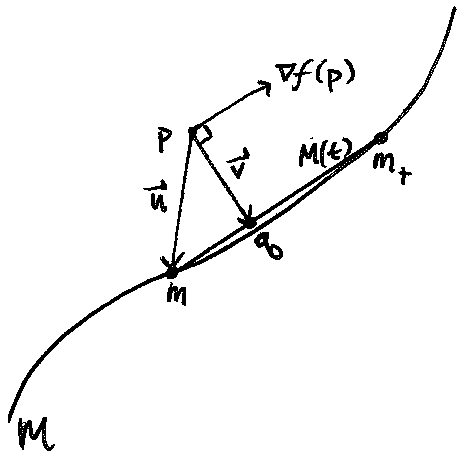}
\caption{Schematic for Level Set Algorithm in $\mathbb{R}^2$ when starting point $p$ is one step from manifold. Notation in schematic matches pseudo-algorithm above. }
\end{figure}

\subsubsection{Algorithm Notes}
\begin{enumerate}
\item  One may be enticed to minimize a distance function $D$, for example
$$ D(t) = \| x(t) - p \|^2$$
by letting
\begin{align*}
0       &= \frac{\partial D(t)}{\partial t}\\
	&= 2 \ \langle x(t) - p,\  \frac{d}{dt}(x(t) - p) \rangle \\
\end{align*}
but this is computationally inefficient because now we must compute $\frac{d}{dt}x(t)$ along with $x(t) - p$. Instead, we recommend computing $\| x(t) - p \|^2$ and searching for the minimum or using some other nearest neighbor search.
\item If $\vv{u} = [u_0, \ldots , u_n]^T$, it is convenient to express $\vv{v} = \vv{u} - \langle u, \nabla f(p) \rangle \nabla f(p).
$$\hat{e_0} = \nabla f(p) / \| \nabla f(p) \|$, and $\{\hat{e_0}, \ldots, \hat{e_n}\}$ is an orthonormal basis for $\mathbb{R}^n$, it is helpful to express

\item It is convenient to let
$$M(t) = (m_{i+1} - m)t + m, \ t \in [0, 1]$$
\item The point on $M(t)$ for which $M(t) - \vv{u}$ is orthogonal to $\nabla f (p)$ can be determined by solving for t in
$$\langle (m_{i+1} - m)t + m - \vv{u}, \hat{e_0} \rangle$$.
Solving for $t$ gives
$$t = \frac{\sum_{k = 1}^n(p_k - m_{i,k})\hat{e_{0,k}}}{\sum_{l = 1}^n(m_{i+1,l} - m_{i,l})\hat{e_{0,l}}}$$
\end{enumerate}

\subsection{Empirical Results}
For proof of concept and comparison to methods in $\cite{constantine2015active}$, we proceed with data synthesized from two functions,
$$f_1(x, y) = \exp{y - x^2}$$
$$f_2(x, y) = x^3 + y^3 + 0.2x + 0.6y.$$
For each example, we are interested in how well two functions, one fit to the the Active Subspace (AS) and one fit to the Active Manifold (AM), recover the values of the function for points outside of the AS and AM. We build the AS and AM and calculate the average $L^1$ error for a set of random test points.
The following experimental set up was observed for each example.
\begin{enumerate}
\item Define a uniform grid on $[-1, 1]^2$ with 0.05 point spacing.
\item Evaluate the gradient of each function, computed analytically, at each grid point.
\item Build the AS and AM using the gradient.
\item Fit the AS and AM with a polynomial ($\hat{f}_1$ quartic, $\hat{f}_2$ quintic).
\item Draw 100 random samples $p \in [-1, 1]^2$  and map them to the AS and AM.
\item Evaluate $f$ at each sample point and $\hat{f}$ at the corresponding projection point.
\end{enumerate}
Upon completing the experiments, average absolute errors between $f$ and $\hat{f}$ were computed.
\begin{center}
\begin{tabular}{ | c | c | c | }
\hline
& $f_1$ & $f_2$ \\
\hline
Manifold $L^1$ Error &  $2.601 \times 10^{-2}$  \ & $7.428 \times 10^{-2} $ \\
\hline
Subspace $L^1$ Error & $1.486 \times 10^{-1}$ \ & $2.051 \times 10^{-1}$ \\
\hline
\end{tabular}
\end{center}
Notice that the AM reduces average absolute error, by an order of magnitude, in both examples.

In \cite{constantine2015discovering} the authors investigate an active subspace in a 5-dimensional single-diode solar cell model. We have followed by reproducing their results with their data, while also implementing the Active Manifold algorithm. Again, we are interested in comparing estimation errors. The experiment included 10,000 randomly sampled points from $[-1, 1]^5$. This sample was partitioned into two randomly ordered sets containing 8,000 and 2,000 points, used for training and testing, respectively.
After an eight fold Monte Carlo simulation, the mean of the average absolute error was computed.
\begin{center}
\begin{tabular}{ | c | c | }
\hline
Manifold $L^1$ Error &  $4.150 \times 10^{-2}$\\
\hline
Subspace $L^1$ Error &  $2.831 \times 10^{-1}$\\
\hline
\end{tabular}
\end{center}
Again, the AM reduces average absolute error, by an order of magnitude

%% file: Conclusions.tex
\section{Conclusions and Future Work}
The primary issue to be addressed in future work is determining the most appropriate way to choose an active manifold. What happens when a level set never intersects the active manifold? Should a new manifold be chosen? For future work, we propose a parallel program implementation that may run the algorithm for two or more manifolds. Future work will also seek error bounds and computational complexity estimates for the algorithm.

%% file: Acknowledgements.tex
\section*{Acknowledgements}

\begin{itemize}
\item This work was supported in part by the U.S. Department of Energy, Office of Science, Office of Workforce Development for Teachers and Scientists (WDTS) under the program SULI.
\item ORNL's Cyber and Information Security Research Group (CISR)

\end{itemize}